\documentclass[12pt, final]{packages/l4dc2022}

\title[RL with almost sure constraints]{Reinforcement Learning with Almost Sure Constraints}
\usepackage{times}
\usepackage{dsfont}
\usepackage{enumitem}
\usepackage{algorithm}
\usepackage{ifthen}
\newboolean{showcomments}
\setboolean{showcomments}{false}
\usepackage{todonotes}

\newboolean{submit-to-conference}
\setboolean{submit-to-conference}{false}

\definecolor{bleudefrance}{rgb}{0.19, 0.55, 0.91}
\definecolor{ao(english)}{rgb}{0.0, 0.5, 0.0}

\newcommand{\addcite}[0]{\ifthenelse{\boolean{showcomments}}
{\textcolor{purple}{(add cite(s)) }}{}}%

\newcommand{\enrique}[1]{  \ifthenelse{\boolean{showcomments}}
{\todo[inline,color=bleudefrance]{Enrique: #1}}{}}
\newcommand{\emmargin}[1]{\ifthenelse{\boolean{showcomments}}{\marginpar{\color{bleudefrance}\tiny EM: #1}}{}}
\newcommand{\juan}[1]{  \ifthenelse{\boolean{showcomments}}
{\todo[inline,color=pink]{Juan: #1}}{}}
\newcommand{\agustin}[1]{  \ifthenelse{\boolean{showcomments}}
{\todo[inline,color=purple!50]{Agustin: #1}}{}}
\newcommand{\hancheng}[1]{  \ifthenelse{\boolean{showcomments}}
{\todo[inline,color=orange]{Hancheng: #1}}{}}
\newcommand{\hmfnote}[1]{  \ifthenelse{\boolean{showcomments}}
{\footnote{{\color{orange} Hancheng: #1}}}{}}

\newboolean{showedits}
\setboolean{showedits}{true}
\usepackage[markup=underlined]{changes}
\definechangesauthor[color=bleudefrance]{EM}
\newcommand{\aem}[1]{
\ifthenelse{\boolean{showedits}}
{\added[id=EM]{#1}}
{\!#1\hspace{-4.75pt}}
}
\newcommand{\repem}[2]{
\ifthenelse{\boolean{showedits}}
{\replaced[id=EM]{#1}{#2}}
{\!#1\hspace{-4.75pt}}
}
\newcommand{\dem}[1]{
\ifthenelse{\boolean{showedits}}
{\deleted[id=EM]{#1}}
{}
}

\DeclareMathOperator*{\argmin}{arg\,min}

\usepackage{wrapfig}

\author{%
\Name{Agustin Castellano} \Email{acaste11@jhu.edu}\\
\Name{Hancheng Min} \Email{hanchmin@jhu.edu}\\
\addr Johns Hopkins University, Baltimore, MD, USA
\AND
\Name{Juan Bazerque} \Email{jbazerque@fing.edu.uy}\\
\addr Universidad de la Rep\'ublica, Montevideo, Uruguay
\AND
\Name{Enrique Mallada} \Email{mallada@jhu.edu}\\
\addr Johns Hopkins University, Baltimore, MD, USA
}

\begin{document}
\maketitle
\begin{abstract}%
 In this work we address the problem of finding feasible policies for Constrained Markov Decision Processes under probability one constraints. We argue that stationary policies are not sufficient for solving this problem, and that a rich class of policies can be found by endowing the controller with a scalar quantity, so called budget, that tracks how close the agent is to violating the constraint. We show that the minimal budget required to act safely can be obtained as the smallest fixed point of a Bellman-like operator, for which we analyze its convergence properties. We also show how to learn this quantity when the true kernel of the Markov decision process is not known, while providing sample-complexity bounds.~
 The utility of knowing this minimal budget relies in that it can aid in the search of optimal or near-optimal policies by shrinking down the region of the state space the agent must navigate. Simulations illustrate the different nature of probability one constraints against the typically used constraints in expectation.
\end{abstract}
\begin{keywords}%
Constrained MDPs, Safe RL, RL for physical systems, sample-efficient learning.
\end{keywords}
\section{Introduction}
With the huge availability of data made possible by cheap sensors and widespread telecommunications, the control paradigm has shifted: the previous-century approach, which relied heavily on system modeling followed by careful control design is now moving towards an \textit{improve-as-you-go} approach, in which controllers are refined on a step by step basis as more data becomes readily available. One of the main tools aiding in the design of these controllers is Reinforcement Learning (RL) \citep{sutton2018reinforcement, bertsekas-rl}. This relatively new field has seen a rebirth in recent years, obtaining outstanding performance in certain domains, particularly when the algorithms are coupled with deep neural networks \citep{mnih2015human} and tree-search methods \citep{silver2016mastering}.
Notwithstanding, this super-human performance has been mostly obtained on setups where \textit{i)} the domain is virtual, \textit{ii)} transition dynamics are fairly simple and \textit{iii)} training is computationally-intesive. There is huge promise, however, in the potential of RL to be extended to complex real-world tasks such as autonomous transportation or robot manipulation, where \textit{safety is paramount}. 

In the subfield of Safe RL, most of the current corpus relies on adding constraints in expectation to trade-off between the conflicting goals of achieving good performance while satisfying feasibility \citep{geibel2006reinforcement,miryoosefi2019reinforcement}, and commonly used methods rely on primal-dual algorithms that take into consideration both the reward function and the constraints to be met \citep{paternain2019learning, ding2020natural}. These methods, however, typically guarantee feasibility only asymptotically---with a possibly unbounded number of constraint violations during training, something highly undesirable in safety-critical systems. Other approaches include formal verification methods \citep{junges1}, which first deal with computing permissive (ie. feasible) strategies, restricting the actions agents can take at each step \citep{junges2}.

In this work we argue that specifying hard constraints actually aids in the development of controllers, since feasible policies are easy to find. This is similar to what some authors have done in the field of deterministic finite automata, where low-complexity policies can be found rapidly \citep{stefansson}. Once safe policies are learnt---or equivalently, once unsafe states and actions are identified---the search for good performance can be done over a smaller set. For real-world applications with physical systems it is critical to keep track of the number of interactions between the agent and the environment. This has led to a drive to develop sample-complexity bounds \citep{agarwal2020modelbased}, and most recently, sample-complexity bounds for learning policies with zero and bounded constraint violations \citep{dileep1,dileep2}.

\noindent
\textbf{Paper outline}: In Section \ref{sec:problem-formulation} we formulate the problem and illustrate why stationary policies are not sufficient under this setting. Section \ref{sec:adequacy} contains the main results of the paper. We show a bijection between the original MDP and one that tracks---via a quantity called \textit{budget}---how close the agent is to violating the constraint. We show feasible policies can be completely characterized in terms of the minimal required budget, which can be obtained as the solution of a fixed point iteration. This requires, however, knowledge of the transition kernel of the MDP. In Section \ref{sec:sample-complexity} we improve on this result by showing that the budget can be learned if one knows an approximate kernel, and give sample-complexity bounds to construct it. Numerical experiments showing the different nature of our proposed constraint as opposed to the state of the art expectation-based counterparts are presented in Section \ref{sec:experiments}, and we conclude in Section \ref{sec:conclusions}.\\
\ifthenelse{\boolean{submit-to-conference}}{\textbf{Proofs}: We provide proof sketches for all our results. Detailed proofs can be found in the appendix of \cite{paper-w-proofs}.}

\section{Problem formulation}\label{sec:problem-formulation}
{Consider a finite state space, finite action space and infinite horizon Constrained Markov Decision Process (CMDP)} defined as a tuple $\mathcal{M}=\left(\mathcal{S}, \mathcal{A}, p, r, d\right)$ where $\mathcal{S}$ is the set of states, $\mathcal{A}$ is the set of actions, $p$ is the kernel that specifies the conditional transition probability $p(s',r,d|s,a)$, $r\in\mathbb{R}$ is the reward and $d\in\{0,1\}$ is a binary-valued \textit{damage indicator} used to model constraint violations. Consider also a user-specified \textit{total damage budget} $\Delta$. The goal in this case is to achieve the highest return while never allowing more than $\Delta$ units of damage in a single trajectory:
\begin{subequations}
\begin{align}
    \max_{\pi\in\Pi_H} &~\mathbb{E}_\pi\left[\sum_{t=0}^\infty R_{t+1} ~\bigg|~ S_0=s\right] \label{eq:maximize-return}\\
    \text{s.t:}&~ P_\pi\left(\sum_{t=0}^\infty D_{t+1} \leq \Delta ~\bigg|~ S_0=s\right)=1\,, \label{eq:og-constraint}
\end{align}
\label{eq:1st-subeq}
\end{subequations}
where the initial state $s$ is fixed and the maximization is carried over the set of general, history-dependent policies $\Pi_H$. 
{We choose to call this framework a  CMDP in the literal sense (as it is an MDP subject to constraints) even though the probability-one constraint \eqref{eq:og-constraint} deviates from the usual expectation-based ones \citep{altman}.} We assume there is an absorbing termination state such that when the system enters this state it remains there with no further reward. We also assume that the structure of the MDP is such that this state is eventually reached under any policy, a common assumption for stochastic shortest path problems \citep{bertsekas-dp}. 

Recalling that the damage $D_t$ is a binary-valued random variable, in essence the quantity $\Delta$ serves as a \textit{tolerance} to damage. A feasible policy is one that---almost surely---does not allow for more than $\Delta$ total damage along a single trajectory. The harshest case requires $\Delta=0$, in which no damage is allowed. When $\Delta=0$, it can be shown that the safety of a particular state-action pair can be encoded in a barrier function akin to the typical action-value function $Q$, under which feasible policies and high-return policies can be learned in parallel \citep{tac}. 

In the original formulation \eqref{eq:1st-subeq} the maximization is carried out over the broad class of history-dependent policies $\Pi_{H}$.  We define the history at time $t$ as the collection of $(S,A,S',R,D)$ tuples up to time $t$, that is $h_t=(s_0,a_0,r_1,d_1,s_1,\ldots,s_{t-1},a_{t-1},r_{t},d_t,s_t)$. 
Policies in this class induce a probability distribution over the set of actions conditioned on the history, i.e. $\pi(\cdot | h_t):\mathcal{A}\to[0,1]$.\\
The class of general policies is a very large set to work with, with the combination of possible histories growing exponentially as time increases. It is desirable, then, to avoid working with history-dependent policies and restrict the optimization over a simpler class that still attains optimal performance. Generally, the class of stationary policies $\pi(\cdot | s_t)$ is considered, in which the distribution over the actions is just a function of the current state.

It is a well-established fact that for unconstrained problems the stationary policies are \textit{complete}, in the sense that they can fully mimic the expected return obtained by any general, history-dependent policy. This result also carries over to constrained problems where the constraint is cast as the expected value of a sum \citep[Thm 3.1, eq. 3.1]{altman}. Borrowing from \textit{completeness}, we define the notion of \textit{adequacy} below, better suited to the type of constraint \eqref{eq:og-constraint} we are handling.
\begin{definition}[Adequacy]
A set of policies $\Pi$ is adequate if for any history-dependent policy $\pi_h$ that is feasible for \eqref{eq:1st-subeq} there exists a feasible policy $\pi\in\Pi$ such that $\mathbb{E}_\pi\left[\sum_{t=0}^\infty R_{t+1}|S_0=s\right]=\mathbb{E}_{\pi_h}\left[\sum_{t=0}^\infty R_{t+1}|S_0=s\right]$.
\end{definition}
It is easy to check that the set of stationary policies is not adequate for solving \eqref{eq:1st-subeq}, as argued in the following proposition. 
\begin{proposition}[Stationary policies are not adequate for $\mathcal{M}$]\label{lem:non-adequacy}
The set of stationary policies is not adequate for solving \eqref{eq:1st-subeq}.
\begin{proof}
\begin{figure}[!h]
    \centering
    \includegraphics[width=.3\linewidth]{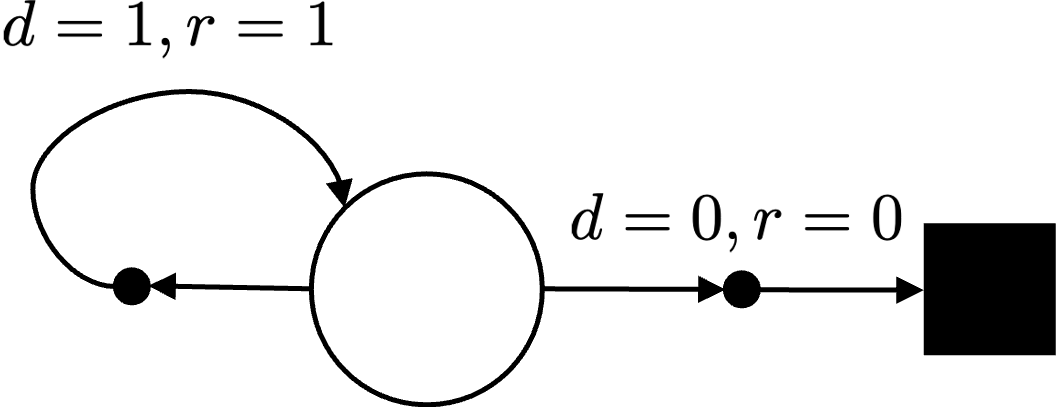}
    \caption{Example MDP: the episode starts in the circle state and ends upon reaching the square.} 
    \label{fig:counterexample}
\end{figure}
As a proof by counterexample, consider the MDP of Figure \ref{fig:counterexample}. The episode starts with the agent in the white circle state and ends when the black square is reached.
The two possible actions are \texttt{left} and \texttt{right}. The optimal policy picks \texttt{left} $\Delta$ times (accruing both reward and damage) and then goes \texttt{right}. It is clear that this policy is non-stationary. Moreover, the only feasible stationary policy is the one that always picks \texttt{right}, obtaining the least return.
The preceding example provides a hint on why general or history-dependent policies work for solving \eqref{eq:1st-subeq}: they keep track of both the rewards (useful for maximization) and the damage incurred so far (needed for feasibility). This fact will be used as a building block towards what will be developed in the next section. Namely, that endowing the controller with memory of the accumulated damage so far is sufficient for learning optimal behavior.
\end{proof}
\end{proposition}



\section{Safe reinforcement learning with memory policies}\label{sec:adequacy}
Throughout this section we argue that in order to learn an optimal policy for \eqref{eq:1st-subeq} it suffices to consider the class of stationary policies that keep track of the accumulated damage along the trajectory. To this end, we consider an augmented MDP with a new state variable $K_t$ that incorporates the accumulated damage so far, which we call budget, and show it to be equivalent to the original MDP. {Tracking this cumulative quantity in the extended MDP is akin to some techniques used to solve unconstrained MDPs under risk-minimization criteria \citep{bauer}.} 
We finalize this section by showing stationary policies in the augmented MDP are adequate.
\begin{definition}[Budget]
For the original MDP $\mathcal{M}$ 
define the budget at time $t$ as the random variable
\begin{equation}
    K_t = \Delta-\sum_{\ell=0}^{t-1}D_{\ell+1}\,,\quad\quad \forall~t\geq 1\,,\label{eq:budget-def}
\end{equation}
with $K_0=\Delta$.
\end{definition}
This term can be seen as the \textit{remaining damage budget}, that is to say, how many more units of damage the agent can suffer along the trajectory while still satisfying \eqref{eq:og-constraint}. From the definition in \eqref{eq:budget-def} it follows that
    $K_{t+1}=K_t-D_{t+1},$
so the budget between successive time steps either stays the same or decreases by one only if damage occurs. We argue that this magnitude $K_t$ is a sufficient statistic for learning an optimal (feasible) policy, in the sense that stationary policies are adequate for a new MDP $\tilde{\mathcal{M}}$ with state variable $\tilde{S}=(S,K)$, which we define next.\\
\begin{definition}[Augmented MDP]
Given transition tuples $(S_t,A_t,S_{t+1},R_{t+1},D_{t+1})$ from $\mathcal{M}$, define the augmented MDP $\tilde{\mathcal{M}}$ as the one with tuples $(\tilde{S}_t,A_t,\tilde{S}_{t+1},R_{t+1},\tilde{D}_{t+1})$ where
\begin{equation}
\tilde{S}_t = (S_t,K_t)\,,\quad\quad\quad\quad \tilde{D}_{t+1}=\mathbf{1}\{K_t-D_{t+1}<0\}\,.
\end{equation}
\end{definition}
In $\tilde{\mathcal{M}}$ the state space is enlarged so as to consider the remaining damage budget $K_t$. Therefore the states $\tilde{S}$ now lie in $\tilde{\mathcal{S}}=\mathcal{S}\times\{\Delta,\Delta-1,\ldots,0\}$.~
The binary damage signal $\tilde{D}_{t+1}$  
is only one when the system is out of budget---i.e., it signifies failure to comply with \eqref{eq:og-constraint}. Figure \ref{fig:augmented-mdp} depicts the structure of this modified MDP. Each blob corresponds to a slice of the state space for fixed $K$, with transitions between states on the same slice occurring as long as the original damage signal $D_{t+1}=0$. When $D_{t+1}=1$ in the original MDP, the transition in the augmented MDP corresponds to decreasing $K_t$ by one. At the slice $\mathcal{S}\times\{0\}$ the system is critically compromised---performing one more unsafe state transition leads to failure (encoded as $\tilde{D}_{t+1}=1$ in the augmented MDP).\\
\begin{figure}[h!]
    \centering
    \includegraphics[width=.5\linewidth]{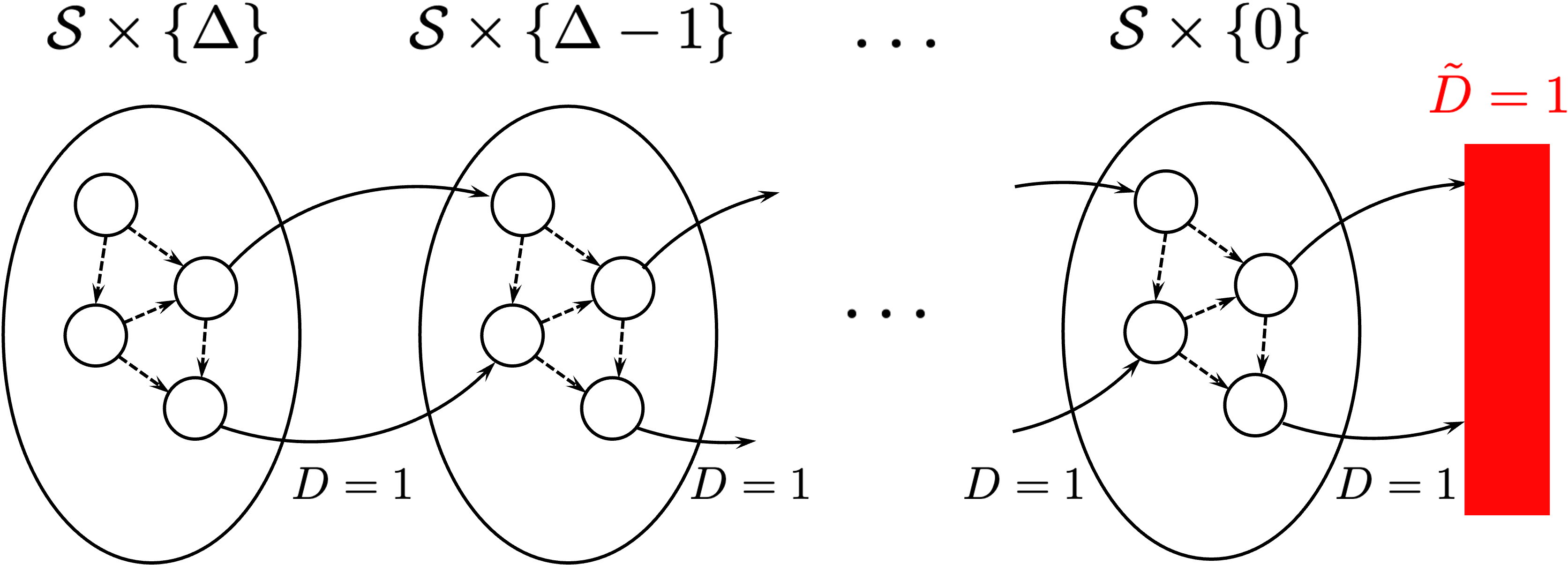}
    \caption{Illustration of transition dynamics in $\tilde{\mathcal{M}}$. Each disk corresponds to a partition of the state-space for fixed budget $K\in\{\Delta,\ldots, 0\}$. At any time step the system either retains the current budget or decreases it by one (when $D_{t+1}=1$), depicted by the solid arrows. Failure occurs when $K_t=0$ and the agent encounters damage ($\tilde{D}_{t+1}=1$ in red).}
    \label{fig:augmented-mdp}
\end{figure}

Consider the following optimization problem on $\tilde{\mathcal{M}}$:
\begin{subequations}
\begin{align}
    \max_{\tilde{\pi}\in\tilde{\Pi}_H} &~\mathbb{E}_{\tilde{\pi}, \tilde{\mathcal{M}}}\left[\sum_{t=0}^\infty R_{t+1} ~\bigg|~ (S_0,K_0)=(s,\Delta)\right] \label{eq:maximize-return-augmented}\\
    \text{s.t:}&~ P_{\tilde{\pi}}\left(\tilde{D}_{t+1} = 0\right)=1\quad\quad\forall t\geq 0\,,\label{eq:augmented-constraint}
\end{align}\label{eq:2nd-subeq}
\end{subequations}
where the first component of the initial state $\tilde{S}_0=(S_0,K_0)$ is the same as in \eqref{eq:maximize-return} and the second component is the total budget $\Delta$ in the original formulation. Maximization in this case is done over the set of history-dependent policies $\tilde{\Pi}_H$, whose elements are of the form $\tilde{\pi}(\cdot|\tilde{h}_t)$ with $\tilde{h}_t=(\tilde{s}_0,a_0,r_1,\tilde{d}_1,\tilde{s}_1,\ldots,\tilde{s}_{t-1},a_{t-1},r_{t},\tilde{d}_t,\tilde{s}_t)$. We explicitly write $\mathbb{E}_{\tilde{\pi},\tilde{\mathcal{M}}}[\cdot]$ in \eqref{eq:maximize-return-augmented} to denote that the expectation is taken with respect to the trajectory induced by $\tilde{\mathcal{M}}$. When there is no room for confusion we use the shorthanded version $\mathbb{E}_{\tilde{\pi}}[\cdot]$ instead.

\subsection{Adequacy of memory policies}
We first argue that the problem in the extended and original MDPs are equivalent. Specifically, any given feasible general policy $\tilde{\pi}_h$ for $\tilde{\mathcal{M}}$ can be readily mapped to a corresponding policy in $\mathcal{M}$ and vice versa. Secondly, we claim that stationary policies are adequate for $\tilde{\mathcal{M}}$. 
In this sense $K_t$ can be seen as a low-complexity descriptor of a policy, gathering all the relevant information in $\tilde{h}_t$.
\begin{lemma}[Equivalence of MDPs]\label{lem:mdp_equivalent}
The optimization problem \eqref{eq:2nd-subeq} in the augmented MDP  $\tilde{\mathcal{M}}$  is equivalent to the optimization problem \eqref{eq:1st-subeq} in the original MDP $\mathcal{M}$ .
\end{lemma}
\begin{proof}[Sketch] We show a bijection $f:\Pi_H\rightarrow \tilde{\Pi}_H$ between the set of history-dependent polices for $\mathcal{M}$ and the one for $\tilde{\mathcal{M}}$ such that the expect return is matched under $f$.
Moreover, $\pi\in\Pi_H$ is feasible for \eqref{eq:og-constraint} if and only if $\tilde{\pi}=f(\pi)$ is feasible for \eqref{eq:augmented-constraint}. Hence the two optimization problems are equivalent.
\end{proof}
\begin{lemma}[Adequacy of stationary policies for $\tilde{\mathcal{M}}$]\label{lem:adequacy}The set of stationary policies of the augmented MDP $\tilde{\mathcal{M}}$ is adequate.
\begin{proof}
Given the fact that $\tilde{D}_{t+1}\in\{0,1\}$ almost surely, \eqref{eq:augmented-constraint} can be equivalently represented as $E_{\tilde\pi}\left[\sum_{t=0}^\infty \tilde{D}_{t+1}\right]=0$. This constraint along with \eqref{eq:maximize-return-augmented} lie in the usual formulation for shortest path problems in CMDPs, for which stationary policies are adequate (Ch.6 in \cite{altman}).
\end{proof}
\end{lemma}

We finish the section by decomposing the action-value function that rises from \eqref{eq:2nd-subeq} in one term focused on return and the other focused on feasibility. Then we show that the feasible stationary policies in $\tilde{\mathcal{M}}$---and therefore the 1-memory policies in $\mathcal{M}$---can be completely characterized by this barrier.
\subsection{Characterizing feasible policies with a barrier function}
Consider for a stationary policy $\tilde{\pi}$ the extended action-value function:
\begin{equation}
    Q_{\tilde{\pi}}(s,k,a) := \mathbb{E}_{\tilde{\pi}}\left[\sum_{\ell=t}^\infty R_{\ell+1}+\mathbb{I}\left\{\sum_{\ell=t}^\infty D_{\ell+1}\leq K_t\right\}~\bigg|~S_t=s, K_t=k, A_t=a\right]\,, \label{eq:q-pi}
\end{equation}
where we introduce the barrier-indicator  $\mathbb{I}\{x\}=0$ if $x$ is true and $\mathbb{I}\{x\}=-\infty$ otherwise.
We can find an optimal policy for \eqref{eq:1st-subeq} by solving
$
    \max_{\tilde{\pi}\in\tilde{\Pi}_S} \mathbb{E}_{a\sim\tilde{\pi}}\left[Q_{\tilde{\pi}}(s,\Delta,a)\right]
$, where $\tilde{\Pi}_S$ is the set of stationary policies in $\tilde{\mathcal{M}}$.
For simplicity, it is useful to specify a function that encodes for the feasibility of the whole state-action space under a policy. This is the  barrier action-value function:
\begin{equation}
    B_{\tilde{\pi}}(s,k,a):=\mathbb{E}_{\tilde{\pi}}\left[\mathbb{I}\left\{\sum_{l=t}^\infty D_{l+1}\leq K_t\right\}~\bigg|~S_t=s, K_t=k, A_t=a\right]~.\label{eq:b-pi}
\end{equation}
This function either takes values zero or $-\infty$, with zero indicating that policy $\tilde{\pi}$ is guaranteed to be feasible when starting from $(s,k,a)$ and $-\infty$ meaning that, with positive probability, more than $k$ units of damage will be seen along the trajectory. This might be a consequence of either 
having too small a budget or a poor policy. The usefulness for defining $B_{\tilde{\pi}}$ is that $Q_{\tilde{\pi}}$ can be decomposed in terms of itself and $B_{\tilde{\pi}}$, which decouples optimality and feasibility, as is shown next.
\begin{lemma}[Barrier decomposition and Bellman equation]\label{lem:decomp-and-bellman}
Let $\tilde{\mathcal{M}}$ be an MDP with an absorbing state 
. Let $\tilde{\pi}$ be a policy in $\mathcal{M}$ such that under $\tilde\pi$ the absorbing state is eventually reached. If rewards $R_{t+1}$ are bounded almost surely for all $t$, then 
\begin{equation}
    Q_{\tilde{\pi}}(s,k,a)=Q_{\tilde{\pi}}(s,k,a)+B_{\tilde{\pi}}(s,k,a)~.
\end{equation}
Additionally, the optimal barrier function $B_*$ satisfies the Bellman equation
\begin{equation}
    B_*(s,k,a)=\mathbb{E}\left[\mathbb{I}\{\tilde{D}_{t+1}\}+\max_{a'\in\mathcal{A}}B_*(S_{t+1},K_{t+1},a')\right]~.\label{eq:bellman-B}
\end{equation}
\end{lemma}
This barrier function satisfies the following monotonicity properties:
\begin{align}
    &B_\pi(s,k,a)=0 \Longrightarrow B_\pi(s,k+i,a)=0\,,&\text{(Safe and more budget $\rightarrow$ safe.)}\label{eq:B-monot-1}\\
    &B_\pi(s,k,a)=-\infty\Longrightarrow B_\pi(s,k-i,a)=-\infty\,.&\text{(unsafe and less budget $\rightarrow$ unsafe.)}\label{eq:B-monot-2}
\end{align}

\subsection{Characterizing feasible policies via minimal budget}
As commented previously, $B_*$ completely characterizes the feasibility of every $(s,k,a)$ triplet. The safety of a  state-action pair $(s,a)$ is conditioned on the agent's remaining budget $k$. With this idea in mind we can define the minimal required budget at each $(s,a)$ as follows.

\begin{definition}[Minimal budget]
The minimal required budget $k_*$ that guarantees feasibility for an $(s,a)$ pair is
\begin{equation}
    k_*(s,a) = \min_{0\leq k\leq\infty} ~k~~\text{s.t.:}~B_*(s,k,a)= 0\,.\label{eq:k_star}
\end{equation}
\end{definition}
This quantity $k_*$ serves as a proxy for safety. A state-action pair 
is safe if the agent's budget $k$ is at least $k_*$, and thus $k_*$ completely characterizes the set of feasible, stationary policies:

\begin{theorem}[Characterization of feasible, stationary policies]\label{thm:feasible-policies}
The set of feasible, stationary policies for \eqref{eq:2nd-subeq} is
\begin{equation}
    \tilde{\Pi}_S^F = \{\tilde{\pi}: \tilde{\pi}(a|s,k)=0\quad\forall a: k_*(s,a)>k\}.
\end{equation}
\begin{proof}
[Sketch] The proof is straightforward and follows from the definition of $k_*$ and monotonicity properties \eqref{eq:B-monot-1}--\eqref{eq:B-monot-2}.
\end{proof}
\end{theorem}
Notice that $k_*$ is intrinsic to the MDP.  We focus on learning this quantity, first establishing a Bellman-like recursion for $k_*$ and then deriving an Algorithm that provably converges to it.
\begin{theorem}[Recursion for $k_*$]\label{thm:k-star}
For each $(s,a)$, the minimal budget satisfies the recursion: 
\begin{equation}
    k_*(s,a)=\max_{s':p(s'\mid s,a)>0}\left[\mathbf{1}_d(s,a,s')+\min_{a'}k_*(s',a')\right]\,,\label{eq:k_star_recursion}
\end{equation}
where $\mathbf{1}_d(s,a,s'):=\mathbf{1}\{p(d=1\mid s,a,s')>0\}$ and $\mathbf{1}\{x\}=1$ if $x$ is true and $0$ otherwise.
\begin{proof}[Sketch.] The proof relies on decomposing $B_*(s,k,a)$ and evaluating it on $k_*(s,a)$.
\end{proof}
\end{theorem}
The recursion in \eqref{eq:k_star_recursion} can be seen as a fixed point of an operator $\mathcal{T}_p$ that acts on budgets $k$ (for a given transition kernel $p(s',d|s,a)$). We define this operator next and analyze some of its properties.

\begin{definition}[The budget operator $\mathcal{T}_p$]
    Given a transition kernel $p$, define the operator $\mathcal{T}_p$ acting on the extended natural vector $\bar{\mathbb{N}}^{\mathcal{S}\times\mathcal{A}}$ with $\bar{\mathbb{N}}=\mathbb{N}\cup\{\infty\}$ as $\mathcal{T}_p:\bar{\mathbb{N}}^{(\mathcal{S}\times\mathcal{A})}\rightarrow \bar{\mathbb{N}}^{(\mathcal{S}\times\mathcal{A})} :$
\begin{equation}
    (\mathcal{T}_p~k)(s,a) := \max_{s':p(s'\mid s,a)>0}\left[\mathbf{1}_d(s,a,s')+\min_{a'}k(s',a')\right]\,, \quad \forall (s,a)\in\mathcal{S}\times\mathcal{A}\,. \label{eq:operator}
\end{equation}
\end{definition}
Notice that here we are making explicit the dependency of $\mathcal{T}_p$ with the transition kernel $p$ from $\mathcal{M}$. 
Later  we will argue that $k_*$ can be learned even if the learner has no access to $p$, as long as it knows a proper surrogate kernel $\hat{p}$. In that case this notation will allow for the difference of $\mathcal{T}_p$ and $\mathcal{T}_{\hat{p}}$. However, for the remainder of this section we spare the subscript and speak just of $\mathcal{T}$.

We would like to make use of this operator to learn $k_*$. The idea is straightforward: for a given kernel $p$, start from $k=0$ and keep applying $\mathcal{T}_p$ until reaching a fixed point. 
But there are many fixed points of \eqref{eq:operator}. Indeed, one can easily check that if $k^\dag$ is a fixed point so is $k^\dag + \mathbf{1}c, c\in\bar{\mathbb{N}}$. Therefore the question still remains as to whether this procedure converges to $k_*$. We will show this is the case, arguing that Algorithm \ref{alg:in-kernel-out-budget} converges to $k_*$.

\RestyleAlgo{boxed,algoruled}
\normalem
\begin{algorithm}
\caption{Fixed point budget iteration}
\KwIn{
Transition kernel $p$ from $\mathcal{M}$}
\KwResult{
$k_*$ for $\mathcal{M}$
}
$k_0(s,a)\leftarrow 0\quad\forall (s,a)\in\mathcal{S}\times\mathcal{A}$\;


\For{$n=0,1,\ldots$}{
\For{$(s,a)\in\mathcal{S}\times\mathcal{A}$}{
    $k_{n+1}(s,a) \leftarrow \max_{s':p(s'\mid s,a)>0}\left[\mathbf{1}\{p(d=1\mid s,a,s')>0\}+\min_{a'}k_n(s',a')\right] $ 
    }
}
\label{alg:in-kernel-out-budget}
\end{algorithm}

\begin{theorem}[Convergence to $k_*$]\label{thm:k-star-convergence}
Define $k_\infty\in\bar{\mathbb{N}}^{\mathcal{S}\times\mathcal{A}}$ as the limit of Algorithm \ref{alg:in-kernel-out-budget}, that is $k_\infty(s,a)=\lim_{n\rightarrow\infty}k_n(s,a)$ for all $(s,a)\in\mathcal{S}\times\mathcal{A}$. Let the input of Algorithm \ref{alg:in-kernel-out-budget} be the true transition kernel $p$ of MDP $\mathcal{M}$. Then the iterates of this algorithm converge to $k_*$:
    $$
    k_\infty(s,a):=\lim_{n\rightarrow\infty}k_n(s,a)=k_*(s,a)\,,\quad\forall (s,a)\in\mathcal{S}\times\mathcal{A}\,.
    $$
\begin{proof}
[Sketch] The proof proceeds by induction, starting from the $(s,a)$ pairs for which $k_*(s,a)=0$.  
\end{proof}
\end{theorem}

Unfortunately, it might be possible that we do not have access to the true kernel $p$. The next section focuses on learning $k_*$ as long as one knows a \textit{surrogate} kernel $\hat{p}$.
\section{Sample complexity for learning the minimal budget}\label{sec:sample-complexity}

We start by defining a consistent kernel $\hat{p}$ of $p$ and show that Algorithm \ref{alg:in-kernel-out-budget} converges to $k_*$ under input $\hat{p}$.
\begin{definition}[Consistent kernel]
Given a transition kernel $p$, $\hat{p}$ is a consistent kernel of $p$ if and only if $\operatorname{sign}\big(\hat{p}(s',d|s,a)\big) = \operatorname{sign}\big({p}(s',d|s,a)\big)~~\forall (s,a,s',d)\in\mathcal{S}\times\mathcal{A}\times\mathcal{S}\times\{0,1\}.$
\end{definition}

Throughout what follows we assume access to a generative model or a sampler, which allows to sample transitions $(s',d)\sim p(\cdot|s,a)$. By collecting $N$ samples at each state-action pair, we can build an empirical model $\hat{p}$ of the transition kernel $p$, counting the fraction of transitions to each $(s',d)$ from $(s,a)$,  as depicted in Algorithm \ref{alg:kernel-builder}. The number of samples needed to build a consistent kernel with arbitrarily high probability is elucidated in Lemma \ref{lem:sample-complexity}.
\begin{algorithm}
\caption{Kernel builder}
\KwIn{Transition kernel $p$ from $\mathcal{M}$, number of sample queries $N$.}
\KwResult{Empirical kernel $\hat{p}$.}
\For{$(s,a)\in\mathcal{S}\times\mathcal{A}$}{
    Sample $N$ transitions $(s',d)\sim p(\cdot|s,a)$
}
Build estimate kernel $\hat{p}(s',d|s,a)=\frac{\operatorname{count}(s',d;s,a)}{N}\quad\forall s'\in\mathcal{S}, d\in\{0,1\}, s\in\mathcal{S}, a\in\mathcal{A}$\;
\label{alg:kernel-builder}
\end{algorithm}
\begin{lemma}[Sample complexity for Algorithm \ref{alg:kernel-builder}]\label{lem:sample-complexity}
Assume that $p(s',d|s,a)=0$ or $p(s',d|s,a)\geq\mu>0$, for every $(s,a,s',d)\in\mathcal{S}\times\mathcal{A}\times\mathcal{S}\times\{0,1\}$. Then with probability at least $1-\delta$,
\texttt{Kernel builder} produces a consistent kernel $\hat{p}$ of $p$, provided that
\begin{equation}
N\geq \frac{1}{\mu}\log\frac{2|\mathcal{S}|^2|\mathcal{A}|}{\delta}\,. \label{eq:sample-comp-consistent-kernel}
\end{equation}
\begin{proof}
[Sketch] Follows from taking a union bound on the probability that Algorithm \ref{alg:kernel-builder} fails to produce a consistent kernel.
\end{proof}
\end{lemma}

It turns out building a consistent kernel is sufficient in order to learn $k^*$ under the true MDP $\mathcal{M}$.
\begin{lemma}[Consistent kernels are enough]\label{lem:consistent-kernels}
    Let $p$ be a transition kernel associated with an MDP $\mathcal{M}$ and let $\hat{p}$ be a consistent kernel of $p$. Then Algorithm \ref{alg:in-kernel-out-budget} with input $\hat{p}$ converges to the minimal budget $k^*$ of $\mathcal{M}$.
\begin{proof}
$\mathcal{T}_{\hat{p}}\equiv\mathcal{T}_{p}$ if $\hat{p}$ is consistent with $p$.
\end{proof}
\end{lemma}

We conclude the main body of the paper with a couple remarks: firstly that the samples needed to learn $k_*$ are small; lastly we discuss the utility of using $k_*$ to learn optimal policies.
\begin{remark}[Learning $k_*$ is sample-efficient.]
    The last two lemmas indicate that the minimal budget $k_*$ can be learned with very few samples. Indeed, the number of interactions with the environment is $\tilde{\mathcal{O}}\left(\frac{|\mathcal{S}||\mathcal{A}|}{\mu}\right)$ disregarding logarithmic terms.  Contrast this, for example, with the $\tilde{\mathcal{O}}\left(\frac{|\mathcal{S}||\mathcal{A}|}{(1-\gamma)^3\epsilon^2}\right)$ dependency needed to learn an $\epsilon$-optimal policy with a generative model \citep{agarwal2020modelbased}. There is no accuracy (i.e. $\epsilon$) requirement in our case, nor the sample complexity depends on the effective horizon $(1-\gamma)^{-1}$: the focus is on detecting transitions rather than estimating them, which makes the problem easier, despite requiring a richer set of policies. 
\end{remark}

\begin{remark}[Using $k_*$ to learn optimal policies]
Throughout this paper we have shown that the minimal budget $k_*$ (which is intrinsic to the MDP) can be efficiently learned. The utility of knowing this quantity is that it characterizes the set of feasible, stationary memory-one policies (as was argued in Theorem \ref{thm:feasible-policies}), or, to put it in another way, it specifies the region of the state space that is $\Delta$-unsafe:
    $$
    \mathcal{S}^\Delta_{\texttt{unsafe}} = \left\{s\in\mathcal{S}: k_*(s,a)>\Delta~\forall a\in\mathcal{A}\right\}\,.
    $$
From this point of view knowing $k_*$ effectively ``trims'' the region of interest of the MDP, meaning the search for an optimal policy is now constrained to a smaller state space, and will therefore require less samples and less computations. 
\end{remark} 
\section{Experiments}\label{sec:experiments}
\begin{figure}[t]
    \centering
    \includegraphics[width=\linewidth]{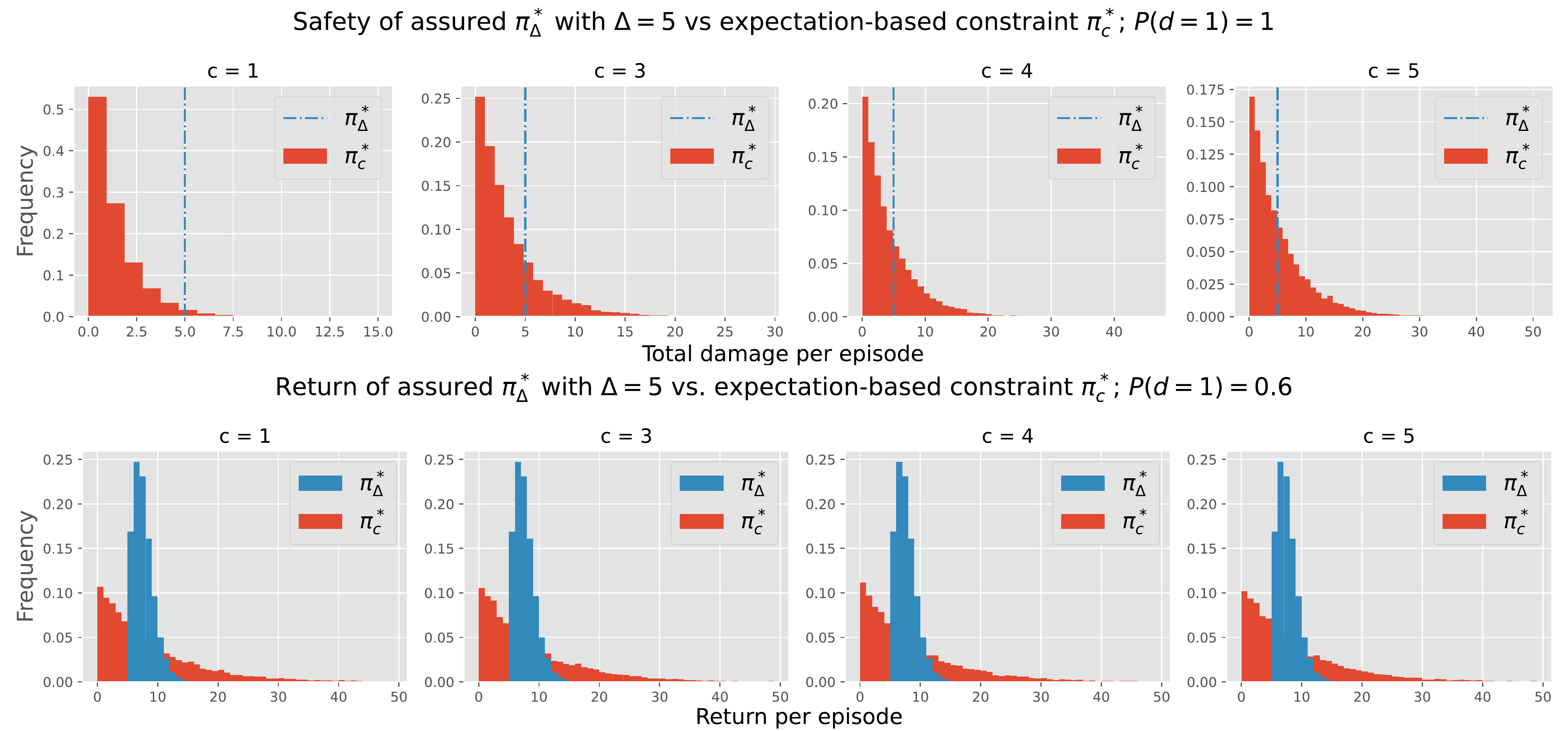}
    \caption{Top: Damage per episode for the optimal policies of the MDP of Figure \ref{fig:counterexample} under different types of constraints. For each panel, the red histogram corresponds to the violations per episode for $\pi_c^*$ under constraint $\mathbb{E}_{\pi_c}\left[\sum_{t=0}^\infty D_{t+1}\right]\leq c$. The assured policy $\pi_\Delta^*$ with \eqref{eq:og-constraint} always attains $\Delta=5$ total damage. Bottom:
    Returns per episode for the optimal policies of the (modified) MDP of Figure \ref{fig:counterexample} where $P(d=1|s,\texttt{left})=0.6$. The policy under our proposed scheme (in blue) always achieves a return of at least $5$, with returns tightly concentrated around $10$.}
    \label{fig:experiment}
\end{figure}
We illustrate the difference between the proposed constraint \eqref{eq:og-constraint} and expectation-like constraints using the  simple MDP of Figure \ref{fig:counterexample}. We recall the optimal policy $\pi_\Delta^*$ under \eqref{eq:og-constraint} chooses action \texttt{left} $\Delta$ times, getting both $\Delta$ damage and reward, and then goes \texttt{right}. The optimal policy $\pi_c^*$ under constraint $\mathbb{E}_{\pi_c}\left[\sum_{t=0}^\infty D_{t+1}\right]\leq c$ chooses action \texttt{left} with probability $\frac{c}{1+c}$ (this is the probability that makes the expected damage constraint hold with equality). It also achieves $c$ expected return.  The top half of Figure \ref{fig:experiment} shows a histogram of the total damage per episode for these optimal policies under different values of $c$, for fixed $\Delta=5$. While the optimal policy for the probability-one constraint $\pi^*_\Delta$ always achieves $\Delta=5$ damage, the damage incurred by the other policy is highly variable. This hints at one of the shortcomings of this type of constraints: even if an optimal policy can be learned, when it is deployed it can have very poor performance in terms of safety.
As a second example, consider a modified version of the same MDP in which the probability of observing damage is $P(d=1|s,\texttt{left})=0.6$. While the optimal policy $\pi_\Delta^*$ remains unchanged, now $\pi_c^*$ takes \texttt{left} more frequently, now with probability $\frac{c}{P(d=1|\texttt{left})+c}$. The bottom half of Figure \ref{fig:experiment} shows the return (sum of rewards in the episode) under both optimal policies. Notice that $\pi_c^*$ is essentially insensitive to $c$, the only difference being that slightly longer tails (not shown on the figure) are observed as $c$ gets larger. The results for the observed damage are similar to those of Experiment I, so we omit them. 
\section{Conclusions}\label{sec:conclusions}
In this work we formulate the problem of Safe Reinforcement Learning under constraints that must be satisfied with probability one. The type of constraint in consideration being that the agent encounters less than $\Delta$ units of damage along a trajectory, where damage is a binary signal. We show that \textit{i)} stationary policies are not adequate for solving this type of problems, \textit{ii)} a sufficiently rich class of policies can be learned if one tracks the damage incurred along the trajectory. The minimal required budget (which is intrinsic to each MDP) can be learned by solving for the fixed point of a newly defined operator, provided one knows a consistent approximation of the transition probabilities. Learning this minimal budget is essentially the same as learning a set of feasible policies. Thus it simplifies the exploration for optimal or near-optimal policies, reducing it to a search within the smaller set of feasible states and actions.
Our experiments illustrate in a simple setup the different nature of probability one constraints as contrasted with expectation-like constraints.

\acks{This work was supported by NSF through grants CAREER 1752362, CPS 2136324, and TRIPODS 1934979, and Johns Hopkins University Catalyst Award, and by ANII-Uruguay through grant FSE\_1\_2019\_1\_159457.}

\clearpage
\bibliography{refs}
\ifthenelse{\boolean{submit-to-conference}}{}{
\clearpage
\section*{Appendix}
\subsection*{Proof of Lemma \ref{lem:mdp_equivalent}}
\begin{proof}
    As we stated in the proof sketch, we show a bijection $f:\Pi_H\rightarrow \tilde{\Pi}_H$ between the set of history-dependent polices for $\mathcal{M}$ and the one for $\tilde{\mathcal{M}}$ such that 
    $\forall\pi\in\Pi_H$, we have $$\mathbb{E}_{\pi,\mathcal{M}}\left[\sum_{t=0}^\infty R_{t+1} ~\bigg|~ S_0=s\right]=\mathbb{E}_{\tilde{\pi},\tilde{\mathcal{M}}}\left[\sum_{t=0}^\infty R_{t+1} ~\bigg|~ (S_0,K_0)=(s,\Delta)\right], \text{where }\tilde{\pi}=f(\pi)\,.$$
    Moreover, $\pi\in\Pi_H$ is feasible for \eqref{eq:og-constraint} if and only if $\tilde{\pi}=f(\pi)$ is feasible for \eqref{eq:augmented-constraint}. 
    
    First, for the MDP $\mathcal{M}$, any $\pi\in\Pi_H$ is determined by set of probability measure on the action space $\{\pi_t( \cdot | h_t): t\geq 0\}$, where we define the history up to time $t$  as $$h_t=(s_0,a_0,r_1,d_1,s_1,\ldots,s_{t-1},a_{t-1},r_{t},d_t,s_t)\in\mathcal{H}_t\,.$$
    Now for the augmented MDP $\tilde{\mathcal{M}}$, the history is defined as
    $$\tilde{h}_t=(s_0,k_0,a_0,r_1,\tilde{d}_1,s_1,\ldots,s_{t-1},k_{t-1},a_{t-1},r_{t},\tilde{d}_t,s_t)\in\tilde{\mathcal{H}}_t\,.$$
    The following mapping $g_t:\mathcal{H}_t\rightarrow \tilde{\mathcal{H}}_t$ is defined from the construction of the augmented MDP
    \begin{align*}
        &\;g(h_t)=\\
        &\;\left(s_0,\overset{k_0}{\overbrace{\Delta-\sum_{l=0}^{-1}d_l}},a_0,r_1,\overset{\tilde{d}_0}{\overbrace{\mathds{1}\{\sum_{l=0}^{-1}d_l> \Delta\}}},s_1,\ldots,s_{t-1},\overset{k_{t-1}}{\overbrace{\Delta-\sum_{l=0}^{t-1}d_l}},a_{t-1},r_{t},\overset{\tilde{d}_{t}}{\overbrace{\mathds{1}\{\sum_{l=0}^{t-1}d_l> \Delta\}}},s_t\right)\,,
    \end{align*}
    This map has well-defined inverse $g^{-1}_t:\tilde{\mathcal{H}}_t\rightarrow \mathcal{H}_t$:
    $$
        g^{-1}(\tilde{h}_t)=\left(s_0,a_0,r_1,\overset{d_1}{\overbrace{k_0-k_1}},s_1,\ldots,s_{t-1},a_{t-1},r_{t},\overset{d_t}{\overbrace{k_{t-1}-k_{t}}},s_t\right)\,.
    $$
    Now for any $\pi=\{\pi_t( \cdot | h_t): t\geq 0\}\in\Pi_H$, let $f(\pi)=\{\pi_t(\cdot| g^{-1}_t(\tilde{h}_t)):t\geq 0\}\in\tilde{\Pi}_H$. $f$ has well-defined inverse $f^{-1}:\tilde{\Pi}_H\rightarrow \Pi_H, f^{-1}(\tilde{\pi})=\{\tilde{\pi}_t(\cdot|g_t(h_t)):t\geq 0\}$, as one can check:
    \begin{align*}
        f\circ f^{-1}(\tilde{\pi})=\{\tilde{\pi}_t(\cdot|g_t^{-1}(g_t(\tilde{h}_t)):t\geq 0\}=\{\tilde{\pi}_t(\cdot|\tilde{h}_t):t\geq 0\}=\pi\,,\\
        f^{-1}\circ f(\pi)=\{\pi_t(\cdot|g_t(g_t^{-1}(\tilde{h}_t)):t\geq 0\}=\{\pi_t(\cdot|h_t):t\geq 0\}=\tilde{\pi}\,.
    \end{align*}
    This shows $f:\Pi_H\rightarrow \tilde{\Pi}_H$ is a bijection. 
    
    Now we left to prove $f$ preserves the expected return and feasibility. It suffices to prove that 
    \begin{equation}\label{eq:hist_dist_equi}
        p_{\pi,\mathcal{M}}(h_t|S_0=s)=p_{f(\pi),\tilde{\mathcal{M}}}(g(h_t)|S_0=s,K_0=\Delta),\ \forall t\geq 0, \forall h_t\in\mathcal{H}_t\,,
    \end{equation}
    that is, the distribution of the history is matched between $\mathcal{M}$ under $\pi$ and $\tilde{M}=f(\pi)$ under $\tilde{\pi}$ through the bijection $g$.
    With \eqref{eq:hist_dist_equi}, we have
    \begin{align*}
        &\;\mathbb{E}_{\pi,\mathcal{M}}\left[\left.\sum_{t=0}^{\infty}R_{t+1} \right\vert S_0=s\right]\\
        &=\;\sum_{t=0}^\infty\mathbb{E}_{\pi_t,\mathcal{M}}\left[\left.R_{t+1} \right\vert S_0=s\right]\\
        &\; =\sum_{t=0}^\infty\int_{\mathcal{H}_{t+1}}r_{t+1}p_{\pi,\mathcal{M}}(h_{t+1}|S_0=s)dh_{t+1}\\
        &\;\qquad \eqref{eq:hist_dist_equi}\\
        &\; =\sum_{t=0}^\infty\int_{\mathcal{H}_{t+1}}r_{t+1}p_{\tilde{\pi},\tilde{\mathcal{M}}}(g(h_{t+1})|S_0=s,K_0=\Delta)dh_{t+1}\\
        &\; =\sum_{t=0}^\infty\int_{\tilde{\mathcal{H}}_{t+1}}r_{t+1}p_{\tilde{\pi},\tilde{\mathcal{M}}}(\tilde{h}_{t+1}|S_0=s,K_0=\Delta)d\tilde{h}_{t+1}\\
        &\;=\mathbb{E}_{\tilde{\pi},\tilde{\mathcal{M}}}\left[\sum_{t=0}^\infty R_{t+1} ~\bigg|~ (S_0,K_0)=(s,\Delta)\right]\,.
    \end{align*}
    Similarly for the constraint, since $\left\{\sum_{t=0}^{M-1} D_{t+1} \leq \Delta\right\}, M\geq 1$ is a decreasing sequence of events, we have, for $\pi\in\Pi_H$ and $\tilde{\pi}=f(\pi)$,
    \begin{align*}
        &\;\mathbb{P}_{\pi,\mathcal{M}}\left(\sum_{t=0}^\infty D_{t+1} \leq \Delta ~\bigg|~ S_0=s\right)\\
        &\;=1-\mathbb{P}_{\pi,\mathcal{M}}\left(\sum_{t=0}^\infty D_{t+1} > \Delta ~\bigg|~ S_0=s\right)\\
        &\;=1-\lim_{M\rightarrow\infty}\mathbb{P}_{\pi,\mathcal{M}}\left(\sum_{t=0}^{M-1} D_{t+1} > \Delta ~\bigg|~ S_0=s\right)\\
        &\;=1-\lim_{M\rightarrow\infty}\int_{\mathcal{H}_{M}}\mathds{1}\left\{\sum_{t=0}^{M-1} d_{t+1} > \Delta \right\}p_{\pi,\mathcal{M}}(h_M|S_0=s)dh_M\\
        &\; (\eqref{eq:hist_dist_equi} \text{, and the definition of } \tilde{d}_t )\\
        &\;=1-\lim_{M\rightarrow\infty}\int_{\mathcal{H}_{M}}\tilde{d}_M p_{\tilde{\pi},\tilde{\mathcal{M}}}(g(h_M)|S_0=s,K_0=\Delta)dh_M\\
        &\;=1-\lim_{M\rightarrow\infty}\int_{\tilde{\mathcal{H}}_{M}}\tilde{d}_M p_{\tilde{\pi},\tilde{\mathcal{M}}}(\tilde{h}_M)|S_0=s,K_0=\Delta)d\tilde{h}_M\\
        &\;=1-\lim_{M\rightarrow\infty}\mathbb{P}_{\tilde{\pi},\tilde{\mathcal{M}}}\left(\tilde{D}_M=1|S_0=s,K_0=\Delta\right)\\
        &\;=1-\sup_{M}\mathbb{P}_{\tilde{\pi},\tilde{\mathcal{M}}}\left(\tilde{D}_M=1|S_0=s,K_0=\Delta\right)\,,
    \end{align*}
    where the last equality is due to the fact that $\tilde{D}_t\leq \tilde{D}_{t+1},\forall t\geq 1$ almost surely, which means $\mathbb{P}_{\tilde{\pi},\tilde{\mathcal{M}}}\left(\tilde{D}_M=1|S_0=s,K_0=\Delta\right)$ is increasing with respect to $M$. Therefore, $\pi$ is feasible, $$\mathbb{P}_{\pi,\mathcal{M}}\left(\sum_{t=0}^\infty D_{t+1} \leq \Delta ~\bigg|~ S_0=s\right)\,,$$
    if and only if
    $$
        \sup_{M}\mathbb{P}_{\tilde{\pi},\tilde{\mathcal{M}}}\left(\tilde{D}_M=1|S_0=s,K_0=\Delta\right)=0\,,
    $$
    or equivalently,
    $$
        \mathbb{P}_{\tilde{\pi},\tilde{\mathcal{M}}}\left(\tilde{D}_t=0|S_0=s,K_0=\Delta\right)=1,\forall t\,,
    $$
    and this exactly means $\tilde{\pi}=f(\pi)$ is feasible. 
    
    With that, we conclude that optimization problem \eqref{eq:maximize-return}--\eqref{eq:og-constraint} is equivalent to
    \begin{align*}
        \max_{f(\pi)\in\tilde{\Pi}_H} &~\mathbb{E}_{f(\pi),\tilde{\mathcal{M}}}\left[\sum_{t=0}^\infty R_{t+1} ~\bigg|~ S_0=s,K_0=\Delta\right] \\
        \text{s.t:}&~ \mathbb{P}_{f(\pi),\tilde{\mathcal{M}}}\left(\tilde{D}_t=0|S_0=s,K_0=\Delta\right)=1,\forall t\,.
    \end{align*}
    This is exactly \eqref{eq:maximize-return-augmented}--\eqref{eq:augmented-constraint}.
    
    What is left to show is equation \eqref{eq:hist_dist_equi}. Suppose $\tilde{\pi}=f(\pi)$. By induction, for $h_0=(s,a_0)$, $g(h_0)=(s,\Delta,a_0)$, we have 
    $$
        p_{\pi,\mathcal{M}}(h_0|S_0=s)=\pi_0(a_0|S_0=s)=\tilde{\pi}_0(a_0|S_0=s,K_0=\Delta)=p_{\tilde{\pi},\tilde{\mathcal{M}}}(g(h_0)|S_0=s,K_0=\Delta)\,,
    $$
    Now assume for some $t\geq 1$, we have
    $$
        p_{\pi,\mathcal{M}}(h_{t-1}|S_0=s)=p_{\tilde{\pi},\tilde{\mathcal{M}}}(g(h_{t-1})|S_0=s,K_0=\Delta)\,,
    $$
    Then $\forall h_t\in\mathcal{H}_t$
    \begin{align*}
        p_{\pi,\mathcal{M}}(h_t|S_0=s)&=\; p_{\pi,\mathcal{M}}(h_t|h_{t-1})p_{\pi,\mathcal{M}}(h_{t-1}|S_0=s)\\
        &=\; p_{\pi,\mathcal{M}}(s_t,a_{t-1},r_t,d_t|h_{t-1})p_{\pi,\mathcal{M}}(h_{t-1}|S_0=s)\\
        &=\; \pi_t(a_{t-1}|h_{t-1})p_{\mathcal{M}}(s_t,r_t,d_t|h_{t-1},a_{t-1})p_{\pi,\mathcal{M}}(h_{t-1}|S_0=s)\,,
    \end{align*}
    For the first term, we have
    \begin{equation}
        \pi_t(a_{t-1}|h_{t-1})=\tilde{\pi}_t(a_{t-1}|g(h_{t-1}))\,.\label{eq:lem_equiv_1}
    \end{equation}
    For the second term, we have
    \begin{align}
        p_{\pi,\mathcal{M}}(s_t,r_t,d_t|h_{t-1})&=\; p_{\mathcal{M}}(s_t,r_t,d_t|s_{t-1},a_{t-1},k_{t-1})\nonumber\\
        &\; (\text{From the construction of augmented MDP})\nonumber\\
        &=\;p_{\tilde{\mathcal{M}}}(s_t,k_{t-1}-d_t,r_t,,\mathds{1}\left\{k_{t-1}-d_t<0\right\}|s_{t-1},a_{t-1},k_{t-1})\label{eq:lem_equiv_2}
    \end{align}
    For the last term, we have, from the induction assumption,
    \begin{equation}
        p_{\pi,\mathcal{M}}(h_{t-1}|S_0=s)=p_{\tilde{\pi},\tilde{\mathcal{M}}}(g(h_{t-1})|S_0=s,K_0=\Delta)\,.\label{eq:lem_equiv_3}
    \end{equation}
    Using \eqref{eq:lem_equiv_1}\eqref{eq:lem_equiv_2}\eqref{eq:lem_equiv_3}, we have 
    \begin{align*}
        &\;p_{\pi,\mathcal{M}}(h_t|S_0=s)\\ 
        &=\; \tilde{\pi}_t(a_{t-1}|g(h_{t-1}))p_{\tilde{\mathcal{M}}}(s_t,r_t,k_{t-1}-d_t,\mathds{1}\left\{k_{t-1}-d_t<0\right\}|s_{t-1},a_{t-1},k_{t-1})\\
        &\;\qquad\qquad\qquad p_{\tilde{\pi},\tilde{\mathcal{M}}}(g(h_{t-1})|S_0=s,K_0=\Delta)\\
        &=\; p_{\tilde{\pi},\tilde{\mathcal{M}}}(s_t,k_{t-1}-d_t,a_{t-1},r_t,\mathds{1}\left\{k_{t-1}-d_t<0\right\}|g(h_{t-1}))p_{\tilde{\pi},\tilde{\mathcal{M}}}(g(h_{t-1})|S_0=s,K_0=\Delta)\\
        &\;=p_{\tilde{\pi},\tilde{\mathcal{M}}}(g(h_t)|g(h_{t-1}))p_{\tilde{\pi},\tilde{\mathcal{M}}}(g(h_{t-1})|S_0=s,K_0=\Delta)\\
        &=\;p_{\tilde{\pi},\tilde{\mathcal{M}}}(g(h_{t})|S_0=s,K_0=\Delta)\,.
    \end{align*}
    Induction on $t$ gives us \eqref{eq:hist_dist_equi}, which completes the proof.
\end{proof}

\subsection*{Proof of \eqref{eq:B-monot-1}--\eqref{eq:B-monot-2} (properties of $B_\pi$)}
\begin{itemize}
\item Safe and more budget $\rightarrow$ safe
$$
    B_\pi(s,k,a)=0 \Longrightarrow B_\pi(s,k+i,a)=0
$$
\begin{proof}
\begin{align*}
    B_\pi(s,k,a)=&\mathbb{E}_\pi\left[-\mathbb{I}\left\{\sum_{l=t}^\infty D_{l+1}\leq k~\bigg|~S_t=s, A_t=a\right\}\right]=0\iff\\
    \iff&P_\pi\left\{\sum_{l=t}^\infty D_{l+1}\leq k~\bigg|~S_t=s, A_t=a\right\}=1
\end{align*}
Event inclusion:
$$
\left\{\sum_{l=t}^\infty D_{l+1}\leq k~\bigg|~S_t=s, A_t=a\right\}\subset\left\{\sum_{l=t}^\infty D_{l+1}\leq k+i~\bigg|~S_t=s, A_t=a\right\}\quad i\geq 0
$$
Monotonicity of P:
$$
1=P_\pi\left\{\sum_{l=t}^\infty D_{l+1}\leq k~\bigg|~S_t=s, A_t=a\right\}\leq P_\pi\left\{\sum_{l=t}^\infty D_{l+1}\leq k+i~\bigg|~S_t=s, A_t=a\right\}\Longrightarrow
$$
$$
B_\pi(s,k+i,a)=0
$$
\end{proof}

\item unsafe and less budget $\rightarrow$ unsafe
$$
B_\pi(s,k,a)=-\infty\Longrightarrow B_\pi(s,k-i,a)=-\infty\quad\quad\text{(unsafe and less slack $\rightarrow$ unsafe.)}
$$
\begin{proof}
$$
B_\pi(s,k,a)=-\infty\iff P_\pi\left\{\sum_{l=t}^\infty D_{l+1}\leq k~\bigg|~S_t=s, A_t=a\right\}<1
$$
Monotonicity:
$$
1>P_\pi\left\{\sum_{l=t}^\infty D_{l+1}\leq k~\bigg|~S_t=s, A_t=a\right\}>P_\pi\left\{\sum_{l=t}^\infty D_{l+1}\leq k-i~\bigg|~S_t=s, A_t=a\right\}\Longrightarrow
$$
$$
B_\pi(s,k-i,a)=-\infty
$$
\end{proof}
\end{itemize}

\subsection*{Proof of Theorem \ref{thm:k-star}}

\begin{proof}
    Consider an $(s,a)$-pair with $k_*(s,a)=K^*\geq 1$. We prove \eqref{eq:k_star_recursion} by contradiction.
    \begin{itemize}
        \item Assume $$K^*=k_*(s,a)>\max_{s':p(s'\mid s,a)>0}\left[\mathds{1}_d(s,a,s')+\min_{a'}k_*(s',a')\right]\,.$$
        This suggests that $\forall s':p(s'\mid s,a)>0$, we have
        $$
            K^*-\mathds{1}_d(s,a,s')>\min_{a'}k_*(s',a')\implies K^*-\mathds{1}_d(s,a,s')-1\geq \min_{a'}k_*(s',a')\,,
        $$
        and which, by definition of $k^*$, it is equivalent to,
        $$
            \mathbb{P}_{\pi^*}\left(\sum_{t=0}^{\infty}D_{t+1}\leq K^*-1-\mathds{1}_d(s,a,s')\mid S_0=s'\right)=1\,.
        $$
        Then we have
        \begin{align*}
            &\;\mathbb{P}_{\pi^*}\left(\sum_{t=0}^{\infty}D_{t+1}\leq K^*-1\mid S_0=s,A_0=a\right)\\
            =&\;\sum_{s'}\mathbb{P}_{\pi^*}\left(\sum_{t=1}^{\infty}D_{t+1}\leq K^*-1-D_1\mid S_0=s,A_0=a,S_1=s'\right)p(s'\mid s,a)\\
            \geq &\; \sum_{s'}\mathbb{P}_{\pi^*}\left(\sum_{t=1}^{\infty}D_{t+1}\leq K^*-1-\mathds{1}_d(s,a,s')\mid S_0=s,A_0=a,S_1=s'\right)p(s'\mid s,a)\\
            =&\;\sum_{s'}\mathbb{P}_{\pi^*}\left(\sum_{t=1}^{\infty}D_{t+1}\leq K^*-1-\mathds{1}_d(s,a,s')\mid S_1=s'\right)p(s'\mid s,a)=1\,.
        \end{align*}
        This is equivalent to $k_*(s,a)\leq K^*-1$, a contradiction. \begin{equation}\label{eq:k_star_leq}
            k_*(s,a)\leq \max_{s':p(s'\mid s,a)>0}\left[\mathds{1}_d(s,a,s')+\min_{a'}k_*(s',a')\right]\,.
        \end{equation}
        \item Assume $$K^*=k_*(s,a)<\max_{s':p(s'\mid s,a)>0}\left[\mathds{1}_d(s,a,s')+\min_{a'}k_*(s',a')\right]\,.$$
        This suggests that there exists an $s'$ with $p(s'\mid s,a)>0$, for which we have
        $$
            K^*-\mathds{1}_d(s,a,s')<\min_{a'}k_*(s',a')\,,
        $$
        and which, by definition of $k^*$, it is equivalent to,
        $$
            \mathbb{P}_{\pi}\left(\sum_{t=0}^{\infty}D_{t+1}> K^*-\mathds{1}_d(s,a,s')\mid S_0=s'\right)>0,\quad \forall\pi\,,
        $$
        and it leads to 
        $$
            \mathbb{P}_{\pi}\left(\sum_{t=0}^{\infty}D_{t+1}> K^*-\mathds{1}_d(s,a,s')\mid S_0=s'\right)p(d=\mathds{1}_d(s,a,s')|s,a,s')>0\,.
        $$
        Then we have for any policy $\pi$
        \begin{align*}
            &\;\mathbb{P}_{\pi}\left(\sum_{t=0}^{\infty}D_{t+1}> K^*\mid S_0=s,A_0=a\right)\\
            =&\;\sum_{s'}\left(\mathbb{P}_{\pi}\left(\sum_{t=1}^{\infty}D_{t+1}> K^*-1\mid S_0=s,A_0=a,S_1=s'\right)p(d=1|s,a,s')p(s'\mid s,a)\right.\\
            &\;\qquad\left.+\mathbb{P}_{\pi}\left(\sum_{t=1}^{\infty}D_{t+1}> K^*\mid S_0=s,A_0=a,S_1=s'\right)p(d=0|s,a,s')p(s'\mid s,a)\right)\\
            =&\;\sum_{s'}\left(\mathbb{P}_{\pi}\left(\sum_{t=1}^{\infty}D_{t+1}> K^*-1\mid S_1=s'\right)p(d=1|s,a,s')p(s'\mid s,a)\right.\\
            &\;\qquad\left.+\mathbb{P}_{\pi}\left(\sum_{t=1}^{\infty}D_{t+1}> K^*\mid S_1=s'\right)p(d=0|s,a,s')p(s'\mid s,a)\right)>0\,.\\
        \end{align*}
        This is equivalent to $k_*(s,a)>K^*$, a contradiction. Therefore one must have
        \begin{equation}\label{eq:k_star_geq}
            k_*(s,a)\geq \max_{s':p(s'\mid s,a)>0}\left[\mathds{1}_d(s,a,s')+\min_{a'}k_*(s',a')\right]\,.
        \end{equation}
    \end{itemize}
    Combining the two inequalities \eqref{eq:k_star_leq}\eqref{eq:k_star_geq}, we have
    $$
        k_*(s,a)= \max_{s':p(s'\mid s,a)>0}\left[\mathds{1}_d(s,a,s')+\min_{a'}k_*(s',a')\right]\,,
    $$
    for $k_*(s,a)=K^*\geq 1$. 
    
    Lastly, for the case $k_*(s,a)=0$, repeat the proof for \eqref{eq:k_star_geq}, we have
    $$0=k_*(s,a)\geq \max_{s':p(s'\mid s,a)>0}\left[\mathds{1}_d(s,a,s')+\min_{a'}k_*(s',a')\right]\,,$$
    then the right hand side must be $0$, which equals $k_*(s,a)$.
\end{proof}

\subsection*{Properties of $\mathcal{T}$}
\begin{definition}
    For two vectors $k_1,k_2 \in\bar{\mathbb{N}}^{\mathcal{S}\times\mathcal{A}}$, we write $k_1\leq k_2$ if $k_1(s,a)\leq k_2(s,a),\forall (s,a)\in\mathcal{S}\times\mathcal{A}$. Then $(\leq)$ defines a partial order relation on $\bar{\mathbb{N}}^{\mathcal{S}\times\mathcal{A}}$.
\end{definition}
\begin{proposition}[Properties of $\mathcal{T}$]
    The operator $\mathcal{T}$ in \eqref{eq:operator} satisfies the following properties:
    
    \begin{enumerate}
        \item \label{lem:operator-monotonicity} $\mathcal{T}$ preserves the partial ordering $(\leq)$ on $\bar{\mathbb{N}}^{\mathcal{S}\times\mathcal{A}}$:
        $\quad k_1\leq k_2 \implies \mathcal{T}k_1\leq\mathcal{T} k_2\,.$
    \item $k_*$ is a fixed point of $\mathcal{T}$: 
    $\quad \mathcal{T}k_*=k_*\,.$
    
    \end{enumerate}
    \begin{proof}
    \begin{enumerate}
        \item Using that $k_1(s',a')\leq k_2(s',a')\quad\forall(s,a)\in\mathcal{S}\times\mathcal{A}$ the following inequality holds:
        \begin{align*}
        (\mathcal{T}~k_1)(s,a) &= \max_{s':p(s'\mid s,a)>0}\left[\mathds{1}_d(s,a,s')+\min_{a'}k_1(s',a')\right]\\
        &\leq \max_{s':p(s'\mid s,a)>0}\left[\mathds{1}_d(s,a,s')+\min_{a'}k_2(s',a')\right]\\
        &= (\mathcal{T}~k_2)(s,a)
        \end{align*}
        \item This is the result of Theorem \ref{thm:k-star}.
    \end{enumerate}
    \end{proof}    
\end{proposition}

\subsection*{Proof of Theorem \ref{thm:k-star-convergence}}
\begin{proof}
Firstly notice Algorithm \ref{alg:in-kernel-out-budget} starts iterating at $k_0=0$. Since (by definition) $0\leq k_*$, then $k_0\leq k_*$. Since $\mathcal{T}$ preserves partial ordering we have $k_n=\mathcal{T}^{(n)} k_0\leq \mathcal{T}^{(n)}k_*=k_*~\forall n$. This means
\begin{equation}
k_\infty (s,a)\leq k_*(s,a)\quad\forall (s,a)\in\mathcal{S}\times\mathcal{A} \label{eq:kinfty-kstar}
\end{equation}
Through the rest of this proof we show \eqref{eq:kinfty-kstar} holds with equality.\\
We proceed by induction, first by showing 
\vspace{0.2cm}

\noindent
\emph{(Induction base)}: $k_*(s,a)=0\iff k_\infty(s,a)=0$.\\
$(\implies)$ We invoke the monotonicity of $\mathcal{T}$ to show
\begin{align*}
    0\leq k_0 &\implies\; \mathcal{T}^{(n)}0\leq \mathcal{T}^{(n)} k_0 = k_n\,, \forall n\geq 0 &\implies\; 0\leq k_\infty\\
    0=k_0\leq k_* &\implies\; k_n=\mathcal{T}^{(n)}k_0\leq \mathcal{T}^{(n)}k_*=k_*\,, \forall n\geq 0 &\implies\; k_\infty\leq k_*
\end{align*}
Then $0\leq k_\infty\leq k_*$, so for all $(s,a) : k_*(s,a)=0\implies k_\infty(s,a)=0$.\\
$(\impliedby)$ Suppose for some $(s,a)\in\mathcal{S}\times\mathcal{A}$, we have
$$k_\infty(s,a)=\max_{s':p(s'\mid s,a)>0}\left[\mathds{1}\{p(d=1\mid s,a,s')\}+\min_{a'}k_\infty(s',a')\right]=0$$
Hence for all $s'$ with $p(s'|s,a)>0$, we must have $p(d=1|s,a,s')=0\text{~and~}\min_{a'}k_\infty(s',a')=0$.

Now consider the policy $\pi(\cdot)$ that, at time $t$, takes actions $A_t=\pi(S_t) = \operatorname{argmin}_a k_\infty(S_t,a)$. If $k_\infty(S_t,A_t)=0$, then by our preceding argument, we have
$$
    \mathbb{P}_\pi(D_{t+1}=1|S_t,A_t,S_{t+1})=0,\ \mathbb{P}_\pi(\min_{a}k_\infty(S_{t+1},a)=0|S_t,A_t)=1\,.
$$
That is, given $k_\infty(S_t,A_t)=0$, taking action $A_t$ at $S_t$ guarantees that, with probability $1$, the MDP transitions to a state $S_{t+1}$ with $\min_{a}k_\infty(S_{t+1},a)=0$ and does not incur damage. Furthermore, $\min_{a}k_\infty(S_{t+1},a)=0$ restricts us to take action $A_{t+1}=\pi(S_{t+1})$ with $k_\infty(S_{t+1},A_{t+1})=0$.

Repeating this argument for the entire trajectory $\tau=\{S_t,A_t\}_{t=0}^\infty$ induced by $\pi$ starting from $S_0=s,A_0=a$. Then with probability $1$, we have
$$
k_\infty(S_t,A_t)=0,\ D_{t+1}=0\,, \forall t\geq 0\,.
$$
Then we find a policy $\pi$ such that $\mathbb{P}_\pi\left(\sum_{t=0}^\infty D_{t+1}\leq 0|S_0=s,A_0=a\right)=0$, then $B_*(s,0,a)=0$, which is equivalent to $k_*(s,a)=0$. This completes the induction base. 
\vspace{0.2cm}

\noindent
\emph{(Induction step)}: Now given some $L: 0\leq L < \infty$, we assume $$k_*(s,a)=\ell\iff k_\infty(s,a)=\ell\,,\ \forall\ell\leq L\,,$$ and show $k_*(s,a)=L+1\iff k_\infty(s,a)=L+1$.\\
$(\implies)$ Given some $(s,a)$ with $k_*(s,a)=L+1$, by monotonicity of $\mathcal{T}$, we have $k_\infty(s,a)\leq k_*(s,a)=L+1$. Now suppose $k_\infty(s,a)=\ell$ for some $\ell<L+1$, then by our inductive assumption, one would have $k_*(s,a)=k_\infty(s,a)=l$, contradicting the fact that $k_*(s,a)=L+1$.
\noindent
$(\impliedby)$ We show the final step, namely $k_\infty(s,a)=L+1\implies k_*(s,a)= L+1$. First, we must have $k_*(s,a)\geq L+1$, otherwise by our inductive assumption $k_\infty(s,a)=l$ for some $l\leq L$, which contradicts that $k_\infty(s,a)=L+1$.

Now we show $k_*(s,a)\leq L+1$. Notice that
$$k_\infty(s,a)=\max_{s':p(s'\mid s,a)>0}\left[\mathds{1}\{p(d=1\mid s,a,s')\}+\min_{a'}k_\infty(s',a')\right]=L+1\,.$$

Consider policy $\pi(\cdot)$ that takes actions $A_t=\pi(S_t)=\argmin_a k_\infty(S_t,a)$ and trajectories $\tau=\{S_t,A_t\}_{t=0}^\infty$ such that $S_0=s, A_0=a$. We carefully examine the $k_\infty$ value along the trajectory $\tau$ under $\pi$.  

First of all, with probability one, $k_\infty(S_t,A_t)\geq k_\infty(S_{t+1},A_{t+1}),\forall t\geq 0$. Because
\begin{align*}
    &\;\mathds{1}\{p(d=1\mid S_t,A_t,S_{t+1})\}+k_\infty(S_{t+1},A_{t+1})\\
    =&\;\mathds{1}\{p(d=1\mid S_t,A_t,S_{t+1})\}+\min_{a'}k_\infty(S_{t+1},a')\\
    \leq&\; \max_{s':p(s'\mid S_t,A_t)>0}\left[\mathds{1}\{p(d=1\mid S_t,A_t,s')\}+\min_{a'}k_\infty(s',a')\right]\leq k_\infty(S_t,A_t)\,.
\end{align*}

Given a trajectory $\tau=\{S_t,A_t\}_{t=0}^\infty$ realized by $\pi$, let $t_0$ be the first $t\geq 0$ such that $L+1=k_\infty(S_t,A_t)> k_\infty(S_{t+1},A_{t+1})$, then $k_\infty(S_{t_0+1},A_{t_0+1})=l$ for some $l\leq L$. For trajectories such that $t_0$ does not exist, we must have that for $t\geq 0$, $L+1=k_\infty(S_t,A_t)=k_\infty(S_{t+1},A_{t+1})$, and $\mathds{1}\{p(d=1\mid S_t,A_t,S_{t+1})\}=0$, i.e. this trajectory incurs no damage at all. Now we consider trajectories such that $t_0$ exists.

Prior to $t_0$, for $t<t_0$, we must also have $\mathds{1}\{p(d=1\mid S_t,A_t,S_{t+1})\}=0$, since $L+1=k_\infty(S_t,A_t)=k_\infty(S_{t+1},A_{t+1})$. At $t_0$, the MDP could incur damage depending on the value of $\mathds{1}\{p(d=1\mid S_{t_0},A_{t_0},S_{t_0+1})\}$.

After $t_0$, by our inductive assumption, we have $k_*(S_{t_0+1},A_{t_0+1})=k_\infty(S_{t_0+1},A_{t_0+1})=l$, and $k_*(S_{t_0+t},A_{t_0+t})=k_\infty(S_{t_0+t},A_{t_0+t}), t\geq 1$ for the remaining trajectory after $t_0$. That is, upon reaching $(S_{t_0+1},A_{t_0+1})$, the remaining trajectory can be also viewed as a trajectory generated by $\pi^*(s)=\argmin_ak_*(s,a)$, and $\pi^*$ will incur at most $l$ damage, starting from $(S_{t_0+1},A_{t_0+1})$, since $k_*(S_{t_0+1},A_{t_0+1})=l$. 

Overall, this trajectory incurs no more than $l+1\leq L+1$ damage. This holds true for any trajectory generated by $\pi$, i.e.
$$
    \mathbb{P}_\pi\left(\sum_{t=0}^\infty D_{t+1}\leq L+1\mid S_0=s,A_0=a\right)=1\,.
$$
This implies $B^*(s,L+1,a)=0$, which is equivalent to $k_*(s,a)\leq L+1$. Since we have shown $k_*(s,a)\geq L+1$, one finally conclude $k_*(s,a)=L+1$.


By induction we shown that $k_*(s,a)=l\iff k_\infty(s,a)=l,\forall 0\leq l<\infty$. The final step is showing $k_*(s,a)=\infty\iff k_\infty(s,a)=\infty$.\\
$(\impliedby)$ For every $(s,a)$ such that $k_\infty(s,a)=\infty$ apply the monotonicity of $\mathcal{T}$ to get $k_*(s,a)\geq k_\infty(s,a)=\infty$, so it must be $k_*(s,a)=\infty$.\\
$(\implies)$ Suppose $k_*(s,a)=\infty$, then $k_\infty(s,a)$ must be $\infty$. Because $k_\infty(s,a)=l$ for any value other than $\infty$ leads to $k_*(s,a)=l$, a contradiction.
\end{proof}

\subsection*{Proof of Lemma \ref{lem:sample-complexity}}
\begin{proof}
For every $(s,a,s',d)$ with $p(s',d|s,a)\geq\mu>0$, the estimated kernel has $\hat{p}(s',d|s,a)>0$ when transition $(s,a)\rightarrow (s',d)$ is observed at least once in our samples. Let $F_{(s,a)\rightarrow(s',d)}$ be the event of sampling the transition $(s',d)$ from $(s,a)$ (where there are at most $2\mathcal{S}$ possible transitions for fixed $(s,a)$). Let $X_{(s,a)\rightarrow(s',d)} = \mathds{1}\left(F_{(s,a)\rightarrow(s',d)}\right)$. Then the probability that Algorithm \ref{alg:kernel-builder} fails to produce a consistent kernel satisfies:
\begin{align*}
    \mathbb{P}\left(\bigcup_{(s,a)\in\mathcal{S}\times\mathcal{A}}\bigcup_{\substack{(s',d):\\\mathbb{P}(s',d|s,a)>0}}\{F_{(s,a)\rightarrow(s',d)}\}^C\right)&\leq\sum_{(s,a)\in\mathcal{S}\times\mathcal{A}}\sum_{\substack{(s',d):\\ \mathbb{P}(s',d|s,a)>0}}\mathbb{P}\left(X_{(s,a)\rightarrow(s',d)}=0\right)\\
    &=\sum_{(s,a)\in\mathcal{S}\times\mathcal{A}}\sum_{\substack{(s',d):\\ \mathbb{P}(s',d|s,a)>0}}(1-\mu)^N\\
    &\leq 2|\mathcal{S}|^2|\mathcal{A}|(1-\mu)^N\,.\\
    &\leq 2|\mathcal{S}|^2|\mathcal{A}|\exp(N\log(1-\mu))\\
    &\leq 2|\mathcal{S}|^2|\mathcal{A}|\exp(-\mu N)\\
    &= 2|\mathcal{S}|^2|\mathcal{A}|\exp\left(-\log\frac{2|\mathcal{S}|^2|\mathcal{A}|}{\delta}\right)=\delta\,,
\end{align*}
where the last inequality uses the fact that $\log(1-\mu)\leq -\mu$.
\end{proof}

\subsection*{Proof of Lemma \ref{lem:consistent-kernels}}
\begin{proof}
        Notice that Algorithm \ref{alg:in-kernel-out-budget} under input $\hat{p}$ performs updates as 
        $$k_{n+1}(s,a) \leftarrow \max_{s':\hat{p}(s'\mid s,a)>0}\left[\mathds{1}\{\hat{p}(d=1\mid s,a,s')\}+\min_{a'}k_n(s',a')\right]. $$
    If $\hat{p}$ is consistent with $p$ then for all $(s,a)\in\mathcal{S}\times\mathcal{A}$ the outer maximization is carried over the same set of $s'$ and $\mathds{1}\{\hat{p}(d=1\mid s,a,s')=1\iff\mathds{1}\{{p}(d=1\mid s,a,s')$. Then the updates of Algorithm \ref{alg:in-kernel-out-budget} under both inputs are identical. Since this algorithm converges under $p$ to $k^*$ (Theorem \ref{thm:k-star-convergence}), the same is true under $\hat{p}$.
    \end{proof}}



\end{document}